\documentclass[twoside,11pt]{article}

\usepackage{jmlr2e}

\usepackage[utf8]{inputenc} 
\usepackage[T1]{fontenc}    
\usepackage{hyperref}       
\usepackage{url}            
\usepackage{booktabs}       
\usepackage{amsfonts}       
\usepackage{nicefrac}       
\usepackage{microtype}      

\usepackage{latexsym}
\usepackage{bm} 
\usepackage{graphicx}
\usepackage{color}
\usepackage{amsmath,amssymb,amsfonts}
\usepackage{mathtools}
\usepackage{url}
\usepackage{enumitem}
\usepackage[normalem]{ulem}

\newtheorem{cor}{Corollary}
\newtheorem{defn}{Definition}

\usepackage[T1]{fontenc}
\usepackage[libertine,cmintegrals,cmbraces,vvarbb]{newtxmath}
\usepackage[scaled=0.95]{inconsolata}
\usepackage{fix-cm}

\newcommand{\sX}{{\mathcal X}}
\newcommand{\sB}{{\mathbb{B}}}
\newcommand{\sY}{{\mathcal Y}}

\newcommand{\sP}{\ensuremath{\mathcal{P}}}

\newcommand{\bE}{\mathbb{E}} 

\newcommand{\bbrac}[1]{\big[ #1 \big]}
\newcommand{\bbbrac}[1]{\bigg[ #1 \bigg]}
\newcommand{\set}[1]{\left\{#1\right\}}

\newcommand{\norm}[1]{\left\|#1\right\|}
\newcommand{\fcal}{\mathcal{F}}
\newcommand{\zcal}{\mathcal{Z}}
\newcommand{\xcal}{\mathcal{X}}
\newcommand{\ycal}{\mathcal{Y}}
\newcommand{\rbb}{\mathbb{R}}
\newcommand{\hcal}{\mathcal{H}}

\newcommand{\bmf}{\mathbf{f}}
\usepackage{float}

\firstpageno{1}

\begin{document}

\title{A Generalization Error Bound for Multi-class Domain Generalization}

\author{\name Aniket Anand Deshmukh \email andeshm@microsoft.com \\
       \addr Bing Ads\\
       Microsoft AI \& Research\\
       Sunnyvale, CA 94085, USA
       \AND
       \name  Yunwen Lei \email yunwen.lei@hotmail.com \\
       \addr Department of Computer Science and Engineering\\
       Southern University of Science and Technology\\
       Shenzhen 518055, China
       \AND
       \name  Srinagesh Sharma \email srinag@umich.edu \\
       \addr Cortana\\
       Microsoft Search, Assistant \& Intelligence\\
       Bellevue, WA 98005, USA
        \AND
       \name  Urun Dogan \email udogan@microsoft.com \\
       \addr Bing Ads\\
       Microsoft AI \& Research\\
      Sunnyvale, CA 94085, USA
        \AND
       \name  James W. Cutler \email jwcutler@umich.edu \\
       \addr  Department of Aerospace Engineering\\
       University of Michigan\\
      Ann Arbor, MI 48109, USA
        \AND
       \name  Clayton Scott \email clayscot@umich.edu \\
       \addr   Department of  EECS\\
       University of Michigan\\
      Ann Arbor, MI 48109, USA
       }


\maketitle

\begin{abstract}
Domain generalization is the problem of assigning labels to an unlabeled data set, given several similar data sets for which labels have been provided.
Despite considerable interest in this problem over the last decade, there has been no theoretical analysis in the setting of multi-class classification. In this work, we study a kernel-based learning algorithm and establish a generalization error bound that scales logarithmically in the number of classes, matching state-of-the-art bounds for multi-class classification in the conventional learning setting. We also demonstrate empirically that the proposed algorithm achieves significant performance gains compared to a pooling strategy.
\end{abstract}


\section{Introduction}
Transfer learning, unsupervised domain adaptation, and weakly supervised learning all have the goal of generalizing without access to conventional labeled training data. One particular form of transfer learning that has garnered increasing attention in recent years is {\em domain generalization} (DG). In this setting, the learner is given unlabeled data to classify, and must do so by leveraging labeled data sets from similar yet distinct classification problems. In other words, labeled training data drawn from the same distribution as the test data are not available, but are available from several related tasks. We use the terms  ``task" and ``domain" interchangeably throughout this paper to refer to a joint distribution on features and labels.

Applications of DG are numerous. For example, each task may be a prediction problem associated to a particular individual (e.g., handwritten digit recognition), and the variation between individuals accounts for the variation among the data sets. Domain generalization is needed when a new individual appears, and the only training data come from different subjects.

As another application, below we consider DG for determining the orbits of microsatellites, which are increasingly deployed in space missions for a variety of scientific and technological purposes. Because of randomness in the launch process, the orbit of a microsatellite is random, and must be determined after the launch. Furthermore, ground antennae are not able to decode unique identifier signals transmitted by the microsatellites because of communication resource constraints and uncertainty in satellite position and dynamics. More concretely, suppose $c$ microsatellites are launched together. Each launch is a random phenomenon and may be viewed as a task in our framework. One can simulate the launch of microsatellites using domain knowledge to generate highly realistic training data (feature vectors of ground antennae RF measurements, and labels of satellite ID). One can then transfer knowledge from the simulated training data to label (identify the satellite) the  measurements from a real-world launch with high accuracy. 

\subsection{Formal Problem Statement}
\label{sec:FormalProblemStatement}
Let  $\mathcal{X} $ be the feature space and $\mathcal{Y} $ the label space  with \(|\sY| = c\). Denote by $\mathcal{P}_{ \mathcal{X} \times \mathcal{Y}} $ the set of probability distributions on $\mathcal{X} \times \mathcal{Y}$ and $\mathcal{P}_{ \mathcal{X}}$ the set of probability distributions on  $\mathcal{X}$. Furthermore, let $\mu$ be  a probability measure on $\mathcal{P}_{ \mathcal{X} \times \mathcal{Y}} $, i.e., whose realizations are distributions on  $\mathcal{X} \times \mathcal{Y}$.

With the above notations, domain generalization is defined as follows. We are given training data sets $ S_i = ((X_{ij}, Y_{ij}))_{1 \leq j \leq n_i} $ such that \((X_{ij},Y_{ij}) \sim P_{XY}^i\) and \(P_{XY}^i \sim \mu\). The test data set is
$S^T = ((X_{j}^T, Y_j^T))_{1 \leq j \leq n_T}$ such that $(X_j^T, Y_j^T) \sim P^T$ and \(P^T \sim \mu\). We assume all $(X,Y)$ pairs are drawn iid from their respective distributions, and that $P_1, \ldots, P_N, P^T$ are iid from $\mu$. The $Y_j^T$ are not visible to the learner, and the goal is to accurately predict \((Y_j^T)_{ 1\leq j\leq n_T }\). For any predicted estimate of a label \(\hat{Y}\), the accuracy is evaluated using a loss function \(\ell :\sY \times \sY \rightarrow \mathbb{R}_+\). For greater flexibility in the multiclass case ($c > 2)$, the label space for prediction is relaxed to \(\mathbb{R}^c\) and a surrogate loss function \(\ell: \mathbb{R}^c \times \sY \rightarrow \mathbb{R}_+\) is employed.

As argued by \citet{blanchard:11:nips}, DG can be viewed as a conventional supervised learning problem where the input to the classifier is the extended feature space \(\sP_{\sX}  \times \sX\). A decision function is a function $\bmf: \mathcal{P}_{ \mathcal{X}} \times  \mathcal{X} \rightarrow \mathbb{R}^c$ that predicts $\hat{Y}_j^T = \bmf(\hat{P}_X^T,X_j^T) $, where $\hat{P}_X$ is the associated empirical distribution. The decision function can be separated into its components $\bmf = (f_1,\ldots,f_c) $ such that \(f_m: \mathcal{P}_{ \mathcal{X}} \times  \mathcal{X} \rightarrow \mathbb{R},\) for \( m = 1,2, ... c\).
We define the empirical training error as
\begin{equation}
\widehat{\varepsilon}(\bmf) = \frac{1}{N}\sum_{i=1}^N \frac{1}{n_i}\sum_{j=1}^{n_i} \ell (\bmf(\widehat{P}_{X}^i,{X}_{ij}), Y_{ij}),
\end{equation}
and by denoting $\widetilde{X} = (P_X,X) $, the risk of a decision function with respect to (w.r.t.) loss $\ell$ as
\begin{equation}
 \varepsilon (\bmf) = E_{P_{XY}^T \sim \mu}  E_{(X^T,Y^T) \sim P_{XY}^T} \ell (\bmf( P_X^T ,X^T),Y^T) = E_{P_{XY}^T \sim \mu}  E_{(X^T,Y^T) \sim P_{XY}^T} \ell (\bmf(\widetilde{X}^T),Y^T).
\end{equation}
The goal of DG is to learn an $f$ that minimizes this risk.

\textbf{Remarks:} (1) Although the risk assumes that the predictor has access to $P_X$, $P_X$ is only known through the empirical marginal $\hat{P}_X$.  at training time as well as at test time. (2) Despite the similarity to standard classification in the infinite sample case, the learning task here is different, because the realizations $(\widetilde{X}_{ij},Y_{ij})$  are neither independent nor identically distributed.
(3) Examples of loss functions \(\ell\) can be found in \citet{lee2004multicategory}, \citet{crammer2001algorithmic} and \citet{weston1998multi}. For detailed discussion on different multiclass loss functions and their general forms see \citet{dogan2016unified,tewari2007consistency,ramaswamy2016convex}.

\subsection{Related Problems}
DG is one of several different learning problems that seek to transfer learnt behavior across domains/tasks, including multi-task learning, learning to learn, and domain adaptation. In multi-task learning, there are several related prediction tasks, and the goal is to leverage similarity between tasks to improve performance on each of the given tasks. Thus, multi-task learning is not concerned with generalization to a new task. The problem of learning to learn (also known as meta-learning or lifelong learning) {\em is} concerned with generalization to a new task, but assumes access to labeled data for that new task. The goal here is to improve the sample complexity of a learning algorithm on the new task by leveraging the training tasks \citep{baxter:2000:jair, maurer:2009:ml, maurer:2013:icml, pentina:2015:alt}.

In domain adaptation, the goal is to make predictions on a target domain, given labeled data from one or more related source domains.
Domain adaptation problems come in two flavors: semi-supervised and unsupervised. In semi-supervised DA, in addition to labeled source data, some limited labeled data is also available in the target domain \citep{donahue2013semi}. In unsupervised DA, labeled data is not available in the target domain. Thus, multi-source unsupervised DA has the same data available as DG, and methods for one of these problems can be applied to the other.  However, these problems are nonetheless different in two important ways. The first difference is that DA and DG have different goals. DA seeks to attain the Bayes risk on the target task, which is considered fixed. In contrast, DG seeks the best performance on a new, previously unseen task, which is considered random. Thus, the two problems have different notions of risk, with DG's risk being the larger of the two (see Lemma 3.1 of \citet{blanchard:11:nips}). To attain optimal performance in DA, it is necessary to make assumptions relating the source and target domains, such as covariate shift, target shift, and others \citep{mansour2009domain, redko2019optimal,courty2016optimal,zhang2013covariate}. In contrast, it is possible to attain optimal risk in DG (asymptotically) without imposing any assumptions on how the different tasks are related. The second difference is that methods for domain generalization do not assume access to the unlabeled testing data at training time. Thus, they do not need to be retrained when a new task is presented.

\subsection{Prior Work on Domain Generalization}
To our knowledge, the problem of domain generalization was proposed and first analyzed by \citet{blanchard:11:nips}. They introduced the notion of risk mentioned previously, proposed a kernel-based learning algorithm, and established a generalization error bound. The present paper effectively extends the work of \citet{blanchard:11:nips} to a multi-class setting. Since then, there has been very little theoretical work on DG. \citet{carbonell:2013:ml} study an active learning variant of DG but in a different probabilistic setting and under the restrictive assumption that the tasks are {\em realizable}. \citet{muandet2013domain} develop a feature extractor for DG for which they state a generalization error guarantee. Their analysis builds on that of \citet{blanchard:11:nips} and thus does not accommodate multi-class losses.

Several other empirically supported approaches to DG have also been proposed \citep{xu2014exploiting, grubinger:2015:iwann, ghifary:2017:pami, motiian2017unified, li2017deeper, li2017learning, li2018deep}. Many of these involve using neural networks to learn a common feature space for all tasks. It is worth noting that such methods are complementary to the one studied by \citet{blanchard:11:nips} and the present paper. Indeed, the kernel-based learning algorithm may be applied after having learning a feature representation by another method, as was done by \citet{muandet2013domain}. There is no doubt that feature-learning will lead to improved empirical performance compared to applying the kernel based-approach on the original input space $\sP_{\sX} \times \sX$, which is typically infinite-dimensional. However, since our interest is primarily theoretical, we restrict our experimental comparison to another such algorithm that operates on the original input space, namely, a simple pooling algorithm that lumps all training tasks into a single dataset and learns a support vector classifier.

\subsection{Analysis of Multi-class Classification}

In the conventional supervised learning setting, early performance guarantees for multi-class classification exhibited a linear dependency on the number of classes~\citep{kuznetsov2014multi,koltchinskii2002empirical}. More recently, refined contraction inequalities for vector-valued Rademacher complexities were developed to capture the relationship among different classes, which implies an improved square-root dependency~\citep{lei2015multi,maurer2016vector,cortes2016structured}. Very recently, this square-root dependency was further improved to a logarithmic dependency by structural result on infinity-norm covering numbers~\citep{lei2019data}, owing to a careful analysis of the interactions between different components of the multivariate predictor. Our work may be viewed as extending the algorithm of \citet{blanchard:11:nips} to the multi-class setting, and extending their analysis by incorporating techniques from \citet{lei2019data} to obtain a generalization error bound for DG with logarithmic dependency on the number of classes.

\subsection{Summary of Contributions and Outline}

Our contributions include: (1) Extending the kernel-based approach to DG of \citet{blanchard:11:nips} from the binary classification setting  to multiclass DG; (2) Proving the first known generalization error bound for multiclass DG, which admits a favorable dependency on the number of classes matching the state-of-the-art results in conventional multiclass learning. In particular, the dependency becomes logarithmic for Crammer-Singer loss and multinomial logistic loss. Our generalization error bounds apply to a general learning setting with a Lipschitz loss function and a general $p$-norm constraint to correlate different classes; (3) a scalable implementation based on random Fourier features and experimental demonstration of the method compared to a pooling approach. 

Our analysis follows that of \citet{blanchard:11:nips} in a global respect, in that we decompose the generalization error into two terms: one due to the sampling of domains and one due to the sampling of training examples from the domains. At a more refined level, however, our analysis differs throughout the proof. In particular, our multiclass extension leverages Lipschitz continuity of multiclass loss functions with respect to the infinity-norm, as well as refined Rademacher complexity analysis, to obtain the aforementioned bound with logarithmic scaling in the number of classes.

In section 2 we describe the kernel-based learning algorithm. Section 3 contains our theoretical analysis, and experimental results appear in section 4.

\section{Kernel-Based Learning Algorithm}
The goal of predicting an optimal classifier on the extended feature space can be solved using kernel based algorithms. For a (symmetric positive definite) kernel $k$, let $H_k$ denote its associated reproducing kernel Hilbert space (RKHS) with the associated norm $\|\cdot\|_k$. Let $\bar{k}: (\mathcal{P}_X \times \sX) \times (\mathcal{P}_X \times \sX) \rightarrow \mathbb{R} $ be a symmetric and positive definite kernel on $\mathcal{P}_X \times \mathcal{X}$, whose construction will be described below. Further let $\hat{P}_X^i$ be the empirical distribution for sample $S_i$ corresponding to $X_{ij},j=1,\ldots,n_i$, and let  $\widetilde{X}_{ij} = (\hat{P}_X^i , X_{ij})  $ be the extended data point. We will find a decision function \(\bmf=(f_1,\ldots,f_c) \in H_{\bar{k}}^c:= H_{\bar{k}} \times \cdots H_{\bar{k}}\).
Define
\begin{equation} \label{eq:fhatdef}
\hat{f}_{\lambda} = \operatorname*{arg\,min}_{f \in H_{\bar{k}}^c} \frac{1}{N}\sum_{i = 1}^N  \frac{1}{n_i} \sum_{j = 1}^{n_i} \ell(\bmf(\widetilde{X}_{ij}),Y_{ij}) + \lambda r(\bmf)  ,
\end{equation}
as the empirical estimate of the optimal decision function. Define the regularizer \( r(\bmf) \) as \( r(\bmf) := \|\bmf\|^2_{H_{\bar{k}}^c} := \sum_{m=1}^c \|f_m\|^2_{H_{\bar{k}}}\). The kernel \(\bar{k}\) can be constructed from three other kernels \(k_x, k_x^\prime\) and \(\kappa\). Let $k_x$ and $k_x'$ be kernels on $\sX$. For example, if $ \sX$ is $ \mathbb{R}^d$, $k_x$ and $k_x'$  could be Gaussian kernels. The so-called kernel mean embedding is the mapping $\Phi: \mathcal{P}_{ \mathcal{X}} \to H_{k_x'}$,
\begin{equation}
\Phi(P) = \int_X k_x^\prime(x, \cdot) dP.
\end{equation}
Let $\kappa$ be a kernel-like function on $\Phi(\mathcal{P}_{ \mathcal{X}})$, such as the Gaussian-like function $\kappa(\Phi(P^1_X),\Phi(P^2_X)) = \exp(-\| \Phi(P^1_X) - \Phi(P^2_X)\|^2/2\sigma_\kappa^2)$. Then $\kappa(\Phi(\cdot),\Phi(\cdot))$ is a kernel on $\mathcal{P}_{ \mathcal{X}}$ \citep{CriSte10}, and we can now define the kernel on the extended feature space via a product kernel
\begin{equation}
\bar{k}((P^1_x,X_1), (P_x^2,X_2)) = \kappa(\Phi(P^1_X), \Phi(P^2_X)) k_x(X_1, X_2).
\end{equation}
The empirical estimate of $\Phi$ can be computed for \( \{X_{ij}\}_{1 \leq j \leq n_i} \), \(X_{ij} \sim P_X^i\) as \( \Phi(\hat{P}_X^i) = \frac{1}{n} \sum_{j=1}^{n_i} k_x^\prime(X_{ij}, \cdot) \).
The representer theorem applies in a modified form for the optimization problem \eqref{eq:fhatdef}, which means that the final predictor has the form
$$
\hat{f}_{\lambda}(\hat{P}^T_X,X^T) = \sum_{i = 1}^N \sum_{j = 1}^{n_i} \alpha_{ij} \bar{k}((\hat{P}^i_X,X_{ij}), (\hat{P}^T_X,X^T)).
$$
The algorithm to learn the weights $\alpha_{ij}$ is similar to multiclass extensions of SVMs such as those presented in \citet{lee2004multicategory} applied over the extended feature space.

\section{Generalization Error Analysis}

We make the following assumptions to analyze the generalization error. For any kernel $k$, $\phi_k(x) := k(\cdot,x) \in H_k$ denotes the canonical feature map. 
For any $p\geq1$ and $R>0$, let
\[
\fcal_{p,R}=\big\{\bmf=(f_1,\ldots,f_c)\in H_{\bar{k}}^c:\|\bmf\|_{2,p}\leq R\big\},
\]
where $\|\bmf\|_{2,p}=\big(\sum_{j=1}^{c}\|f_j\|_{\bar{k}}^p\big)^{\frac{1}{p}}$ is the $\ell_p$ norm of $f\in H_{\bar{k}}^c$. For any $\mathbf{a}=(a_1,\ldots,a_c)\in\mathbb{R}^c$, we denote $\|\mathbf{a}\|_\infty=\max_{m=1,\ldots,c}|a_m|$.
\begin{enumerate}[align=left, leftmargin=*, label=\textbf{A \Roman*}]

\item \label{A:loss} The loss function \( \ell: \mathbb{R}^c \times \sY \rightarrow \mathbb{R}_+\) is  is \(L_{\ell} \)-Lipschitz (in the first variable) w.r.t. the infinity norm:  \(
|\ell(\mathbf{a},y) - \ell(\mathbf{b},y)| \leq L_{\ell} \norm{\mathbf{a} - \mathbf{b}}_\infty\)
for \(\mathbf{a}, \mathbf{b} \in \mathbb{R}^c\) and $y$. We also assume $B_Y := \sup_{y\in\sY}\ell(0,y)<\infty$.

\item \label{A:kernelbound} Kernels \(k_x , k^\prime_{x}, \kappa \) are bounded by \(B_{k}^2, B_{k^\prime}^2, B_{\kappa}^2 \) respectively.

\item \label{A:featureholder}  The canonical feature map \(\phi_{\kappa}: H_{k_{x}^\prime} \rightarrow H_{\kappa} \) is \(\alpha\)-H\"{o}lder continuous with $\alpha\in(0,1]$, i.e., \( \forall a,b \in \sB_{k_x^\prime}(B_{k^\prime}): \)
    \[
     \| \phi_{\kappa}(a) - \phi_{\kappa}(b) \|_{\kappa} \leq L_{\kappa} \| a - b\|_{k_x'}^{\alpha}.
    \]

\end{enumerate}

Condition \ref{A:featureholder} holds with $\alpha=1$ when $\kappa$ is the Gaussian-like kernel on $H_{k_x^\prime}$. Using the stated assumptions we shall now develop generalization error bounds for multiclass DG. For the sake of simplicity, we assume that $n_i = n $ to state theoretical results. 

\subsection{Main Results}

Under the above assumptions, we can present the main results on generalization error bounds.

\begin{theorem} \textbf{(Estimation error control)}
\label{thm:EstErrCtrl}
Let $\delta\in(0,1)$.
If conditions \ref{A:loss} - \ref{A:featureholder} hold, then for any $R > 0$, with probability at least \(1 - \delta \)
\begin{multline}\label{estErrCtrl}
    \sup_{\bmf\in\fcal_{p,R}}| \widehat{\varepsilon}(\bmf) - {\varepsilon}(\bmf)| \leq
    L_{\ell}R B_{k}\bigg( L_{\kappa} (4B_{k^\prime})^{\alpha} \big(n^{-1}\log(2N/\delta)\big)^{\frac{\alpha}{2}}\\
     + 54B_\kappa c^{\frac{1}{2}-\frac{1}{\max\{2,p\}}}N^{-\frac{1}{2}}\big(1+\log^{\frac{3}{2}}c\sqrt{2}N\big) + (B_\kappa+B_Y/(L_\ell RB_k))N^{-\frac{1}{2}}\sqrt{2\log(8/\delta)}\bigg),
\end{multline}
\end{theorem}

As a direct corollary, we derive generalization bounds for specific learning machines. Note that both the loss $\ell(\mathbf{a},y)=\max_{j:j\neq y}\big(1-a_y+a_{j}\big)_+$~\citet{crammer2001algorithmic} and $\ell(\mathbf{a},y)=\log\big(\sum_{j=1}^{c}\exp(a_j-a_y)\big)$~\citet{bishop2006pattern} satisfy \ref{A:loss} with $L_\ell=1$, while the loss $\ell(\mathbf{a},y)=\sum_{j=1}^{c}\big(1-a_y+a_j\big)_+$~\citet{weston1998multi} and
$\ell(\mathbf{a},y)=\sum_{j=1,j\neq y}^{c}(1+t_j)_+$~\citet{lee2004multicategory} satisfy \ref{A:loss} with $L_\ell=c$.  We omit the proof for simplicity.
\begin{cor}\label{cor:csml}
  Let $\ell$ be either the Crammer-singer loss function $\ell(\mathbf{a},y)=\max_{j:j\neq y}\big(1-a_y+a_{j}\big)_+$~\citet{crammer2001algorithmic} or the multinomial logistic loss $\ell(\mathbf{a},y)=\log\big(\sum_{j=1}^{c}\exp(a_j-a_y)\big)$~\citet{bishop2006pattern}.
  Let $\delta\in(0,1)$. If conditions \ref{A:kernelbound}, \ref{A:featureholder} hold, then with probability at least \(1 - \delta \) \eqref{estErrCtrl} holds with $L_\ell=1$.
\end{cor}

\begin{cor}
  Let $\ell$ be either the loss function used in Weston and Watkins MC-SVM~\citet{weston1998multi} $\ell(\mathbf{a},y)=\sum_{j=1}^{c}\big(1-a_y+a_j\big)_+$ or the loss function used in Lee et al. MC-SVM~\citet{lee2004multicategory} $\ell(\mathbf{a},y)=\sum_{j=1,j\neq y}^{c}(1+t_j)_+$.
  Let $\delta\in(0,1)$. If conditions \ref{A:kernelbound}, \ref{A:featureholder} hold, then with probability \(1 - \delta \) \eqref{estErrCtrl} holds with $L_\ell=c$.
\end{cor}

\begin{remark}
  We can argue the tightness of the generalization bound in Theorem \ref{thm:EstErrCtrl} as  follows. Its dependency on the number of classes is optimal in the sense that it matches the corresponding ones in the standard learning setting with the test domain identical to the training domain. In particular, it enjoys a logarithmic dependency for Lipschitz continuous loss functions with the associated Lipschitz constant independent of $c$ (Corollary \ref{cor:csml}) and $p\leq 2$. Its dependency on other parameters (e.g., $N$ and $n$) matches the state-of-the-art results for DG in the binary classification setting~\citep{blanchard:11:nips}.
\end{remark}
\begin{remark}
Although our analysis holds for a general $p$-norm constraint expressing the correlation among components of decision function, the specific $p=2$-norm regularization is always used in practical implementations, e.g., Liblinear.
\end{remark}

\subsection{Sketch of Proof}
In this subsection, we sketch the main steps in proving Theorem \ref{thm:EstErrCtrl} and omit some details due to the space limit. The complete proof can be found in the appendix.
Our basic strategy is to decompose the generalization error into two terms: one due to the sampling of distributions from $\mu$ and one due to the sampling of training examples from the sampled distribution.
We need to introduce some useful lemmas. 
The following lemma quantifies the concentration behavior of the empirical average of random variables in Hilbert spaces from their expectation.
\begin{lemma}\textbf{(Hoeffding's Inequality in Hilbert spaces  \citep{steinwart2008support} )}
     \label{thm:Hoeffding_Hilbert}
     Let $(\Omega,\mathcal{A},P)$ be a probability space, $ H $ be a separable Hilbert space, and $B > 0$. Furthermore, let $ \eta_1, ... , \eta_n : \Omega \rightarrow H$ be independent $H$-valued random variables satisfying $ \|\eta_i \|_{\infty} \leq B $ $\forall i = 1 ,..., n$. Then, for all $\delta \in  (0,1) $ with probability at least $1-\delta$
     \[ \Big\| \frac{1}{n} \sum_{i=1}^n (\eta_i - \mathbb{E}_P \eta_i) \Big\|_H \leq B \sqrt{\frac{2 \log(1/\delta)}{n}} + B \sqrt{\frac{1}{n}} + \frac{4B \log(1/\delta)}{3n}. \]
\end{lemma}

Our analysis require to control Rademacher complexities of some function classes composited by a Lipschitz operator over vector-valued function classes. To this aim, we need the following lemma showing how the Rademacher complexity of this composite function class can be bounded in terms of the dimension of the output. Lemma \ref{lem:RC-MC} follows from Theorem 5 and Proposition 7 in~\citet{lei2019data}.
\begin{defn}[Rademacher complexity]\label{def:rademacher}
  Let $\widetilde{\hcal}$ be a class of real-valued functions defined over a space $\widetilde{\zcal}$ and $S'=\{\tilde{z}_i\}_{i=1}^n\in\widetilde{\zcal}^n$.
  The \emph{empirical Rademacher complexity} of $\widetilde{\hcal}$ w.r.t. $S'$ is defined as
  $
    \mathfrak{R}_{S'}(\widetilde{\hcal})=\mathbb{E}_{\epsilon_i}\big[\sup_{h\in \widetilde{\hcal}}\frac{1}{n}\sum_{i=1}^n\epsilon_ih(\tilde{z}_i)\big],
  $
  where $\epsilon_1,\ldots,\epsilon_n$ are independent Rademacher variables, i.e., they take values in $\{+1,-1\}$ with equal probabilities.
\end{defn}
\begin{lemma}[\citep{lei2019data}]\label{lem:RC-MC}
  Let $\widetilde{\zcal}=\widetilde{\xcal}\times\widetilde{\ycal}$ be a input-output space pair and $\widetilde{S}=\{\tilde{z}_1,\ldots,\tilde{z}_m\}\subset\widetilde{Z}^m$. Let $\widetilde{\hcal}$ be a RKHS defined on $\widetilde{\xcal}$ with $\tilde{k}$ be the associated kernel. Let $\widetilde{\fcal}_{p,R}=\{(f_1,\ldots,f_c):f_j\in\widetilde{\hcal},\big(\sum_{j=1}^{c}\|f_j\|_{\tilde{k}}^p\big)^{\frac{1}{p}}\leq R\}$ and $\tilde{\ell}:\rbb^c\times\widetilde{\ycal}\mapsto\rbb^+$ be a Lipschitz function satisfying
  $
  |\tilde{\ell}(\mathbf{a},y)-\tilde{\ell}(\tilde{\mathbf{a}},y)|\leq L\|\mathbf{a}-\tilde{\mathbf{a}}\|_\infty.
  $
  Then, there holds
  \begin{multline*}
  \mathfrak{R}_{\widetilde{S}}\Big(\Big\{z\mapsto\tilde{\ell}\big(\big(f_1(x),\ldots,f_c(x)\big),y\big):(f_1,\ldots,f_c)\in\widetilde{\fcal}_{p,R}\Big\}\Big)\\
  \leq 16L\sqrt{\log2}R\sup_{x\in\widetilde{\xcal}}\sqrt{\tilde{k}(x,x)} m^{-\frac{1}{2}}c^{\frac{1}{2}-\frac{1}{\max\{2,p\}}}\big(1+\log^{\frac{3}{2}}\sqrt{2}mc\big).
  \end{multline*}
\end{lemma}
 \begin{proof}[Proof of Theorem \ref{thm:EstErrCtrl}]

 The function $\bmf$ can be split into c components $\bmf=(f_1,\ldots,f_c)$. It follows from the triangle inequality that
 \begin{align}
 \sup_{\bmf\in\fcal_{p,R}}| \widehat{\varepsilon}(\bmf) - {\varepsilon}(\bmf)|
    &\leq \sup_{\bmf\in\fcal_{p,R}}  \Big|\widehat{\varepsilon}(\bmf)  - \frac{1}{N}\sum_{i=1}^N \frac{1}{n_i}\sum_{j=1}^{n_i} \ell (\bmf({P}_{X}^i,{X}_{ij}), Y_{ij}) \Big| \notag\\
    & +  \sup_{\bmf\in\fcal_{p,R}} \Big|\frac{1}{N}\sum_{i=1}^N \frac{1}{n_i}\sum_{j=1}^{n_i} \ell (\bmf({P}_{X}^i,{X}_{ij}), Y_{ij}) - {\varepsilon}(\bmf)\Big| =: (I) + (II).\label{error-decomposition-main}
 \end{align}

We now control these two terms separately.

 \paragraph{Control of Term (I).} For the first term, one can use the Lipschitz continuity of $\ell$ to derive that
 \begin{align}
 \Big|\widehat{\varepsilon}(\bmf) - \frac{1}{N}\sum_{i=1}^N \frac{1}{n_i}\sum_{j=1}^{n_i} \ell (\bmf({P}_{X}^i,{X}_{ij}), Y_{ij})\Big|
    & \leq L_\ell\max_{l=1,\ldots,c}\max_{i=1,\ldots,N}\max_{j=1,\ldots,n_i}\big|f_l(\hat{P}_X^i,X_{ij})-f_l(P_X^i,X_{ij})\big|. \label{eq:first_term_main}
 \end{align}

 Then we use the reproducing property of the kernel, the definition and the H\"older continuity of the kernel $\kappa$, and  Hoeffding's inequality in the Hilbert space \(H_{k_x^\prime}\) (Lemma \ref{thm:Hoeffding_Hilbert}), to show that with probability $ 1 - \delta $
 \begin{equation}
 \label{eq:first_term_bound_main}
 \sup_{\bmf\in\fcal_{p,R}} \Big|\widehat{\varepsilon}(\bmf) - \frac{1}{N}\sum_{i=1}^N \frac{1}{n_i}\sum_{j=1}^{n_i} \ell (\bmf({P}_{X}^i,{X}_{ij}), Y_{ij}) \Big| \leq   L_{\ell} L_{\kappa} R B_{k}(B_{k^\prime})^{\alpha} \Big( \sqrt{\frac{2 \log \frac{N}{\delta}}{n}} + \sqrt{\frac{1}{n}} + \frac{4 \log \frac{N}{\delta}}{3n} \Big)^{\alpha},
 \end{equation}
 where we have used $\|f_l\|_{\bar{k}}\leq R$ for $\bmf\in\fcal_{p,R}$.

  \paragraph{Control of Term (II).} We now turn to the term (II) in Eqn. \ref{error-decomposition-main}. To this aim, we first consider the term \((II)^\prime\) as the one-sided version of term \((II)\) (i.e., without the absolute value)
    \begin{align}
        (II)^\prime & \leq
        \begin{multlined}[t]
                    \sup_{\bmf\in\fcal_{p,R}} \frac{1}{N}\sum_{i=1}^N \Big(\frac{1}{n_i}\sum_{j=1}^{n_i} \ell (\bmf({P}_{X}^i,{X}_{ij}), Y_{ij})
                        - \bE \bbrac{ \ell( \bmf( \widetilde{X}),Y ) \big| P_{XY}^i }\Big) \notag\\
                    + \sup_{\bmf\in\fcal_{p,R}} \frac{1}{N} \sum_{i=1}^N  \Big( \bE \bbrac{ \ell( \bmf( \widetilde{X}),Y ) \big| P_{XY}^i } - \bE \bbrac{ \ell( \bmf( \widetilde{X}),Y ) } \Big)=: (IIa) + (IIb).\label{II-decomposition-main}
                \end{multlined}
\end{align}

 \paragraph{Control of Term (IIa).} Conditional to \(P_{XY}^1\),...,\(P_{XY}^N\) , the random variables \((X_{ij}, Y_{ij})_{ij}\) are independent (not identically distributed). Introduce the random variable
\begin{equation*}
    \zeta ((X_{ij}, Y_{ij})_{ij}) = \sup_{\bmf\in\fcal_{p,R}} \frac{1}{N}\sum_{i=1}^N \Big(\frac{1}{n_i}\sum_{j=1}^{n_i} \ell (\bmf({P}_{X}^i,{X}_{ij}), Y_{ij}) - \bE \bbrac{ \ell( \bmf( \widetilde{X}),Y ) \big| P_{XY}^i }\Big).
\end{equation*}
It can be checked by the Lipschitz continuity of $\ell$ that $\sup_{\bmf\in\fcal_{p,R}}\ell(\bmf(P,X),Y)\leq B_\ell$, where $B_{\ell}=B_Y+L_\ell  B_{k}B_{\kappa}R$.
Using Rademacher complexity analysis and
applying Lemma \ref{lem:RC-MC}  with $m=Nn$ we obtain
\[
\bE [\zeta | (P^i_{XY})_{1 \leq i \leq N} ] \leq 27L_\ell RB_\kappa B_k \big(Nn\big)^{-\frac{1}{2}}c^{\frac{1}{2}-\frac{1}{\max\{2,p\}}}\big(1+\log^{\frac{3}{2}}\sqrt{2}Nnc\big),
\]
\paragraph{Control of term (IIb).}
Introduce the random variable
\begin{equation*}
    \xi((P_{XY}^i)_{1 \leq i \leq N}) = \sup_{\bmf\in\fcal_{p,R}} \frac{1}{N} \sum_{i=1}^N  \Big( \bE \bbrac{ \ell( \bmf( \widetilde{X}),Y ) \big| P_{XY}^i } - \bE \bbrac{ \ell( \bmf( \widetilde{X}),Y ) } \Big).
\end{equation*}

By symmetrization trick in relating the deviation between empirical means from expectations to Rademacher complexities, and an application of Lemma \ref{lem:RC-MC} shows that
\[
\bE[\xi]\leq 27L_\ell RB_{\kappa}B_k N^{-\frac{1}{2}}c^{\frac{1}{2}-\frac{1}{\max\{2,p\}}}\big(1+\log^{\frac{3}{2}}\sqrt{2}Nc\big).
\]
Combining terms (I), (IIa) and (IIb),  we derive the final bound.
\end{proof}

\section{Results}
We test the proposed algorithm on 4 multiclass datasets and compare it with pooling, where data from all the tasks are pooled together to learn one single classifier. Datasets description are given below and a summary is in Table \ref{tb:datasets}.

\begin{table}[htb!]
\centering
\begin{tabular}{ | c | c | c | c |c |}
\hline
Dataset & Training Tasks & Test Tasks & Examples Per Task & Classes \\ \hline
Synthetic & 80 & 20 & 100 & 10   \\ \hline
Satellite & 400 & 100 & 77-165 & 3  \\ \hline
HAR & 20 & 10 & 300 & 6   \\ \hline
MNIST-MOD & 80 & 20 & 100 & 10   \\ \hline
\end{tabular}
\caption{Summary of Datasets}
\label{tb:datasets}
\end{table}

\textbf{Synthetic Dataset:} Features for synthetic data are drawn from the unit square. Based on one of the dimensions, the data are labeled from 0 to 10, e.g., if the feature value is between 0 and 0.1, then it is labeled as 1, if it is in between 0.1 and 0.2, then it is labeled as 2, and so on. After that, the feature vectors are rotated clockwise by an angle randomly drawn from $0$ to $ 180 $ degrees to get data for one task. The process is repeated 100 times to get data for 100 tasks out of which 80 are train tasks and 20 are test tasks. Fig. \ref{fig:synth1} shows 3 such tasks for $ \theta = 0, 90 $ and $180$ where the supports do not overlap at all, and Fig. \ref{fig:synth2} shows 13 tasks where the supports overlap.

\begin{figure}[htb!]
\begin{minipage}{0.45\linewidth}
  \centering
    \includegraphics[width=0.9\textwidth]{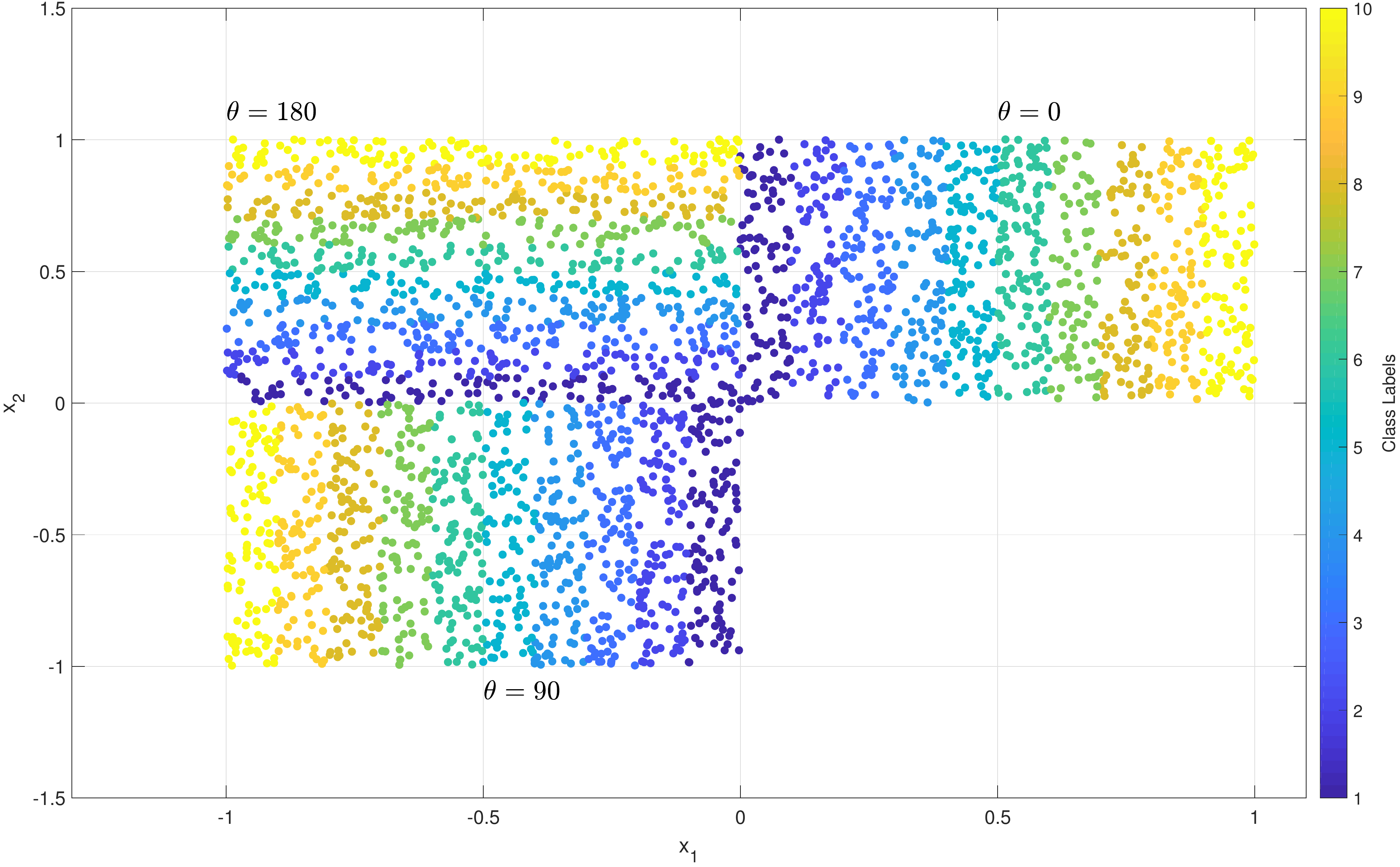}
     \caption{Synthetic Dataset: Three tasks $ \theta = \{0,90,180 \}$ }
     \label{fig:synth1}
\end{minipage}
\hspace{0.5cm}
 \begin{minipage}{0.45\linewidth}
  \centering
    \includegraphics[width=0.9\textwidth]{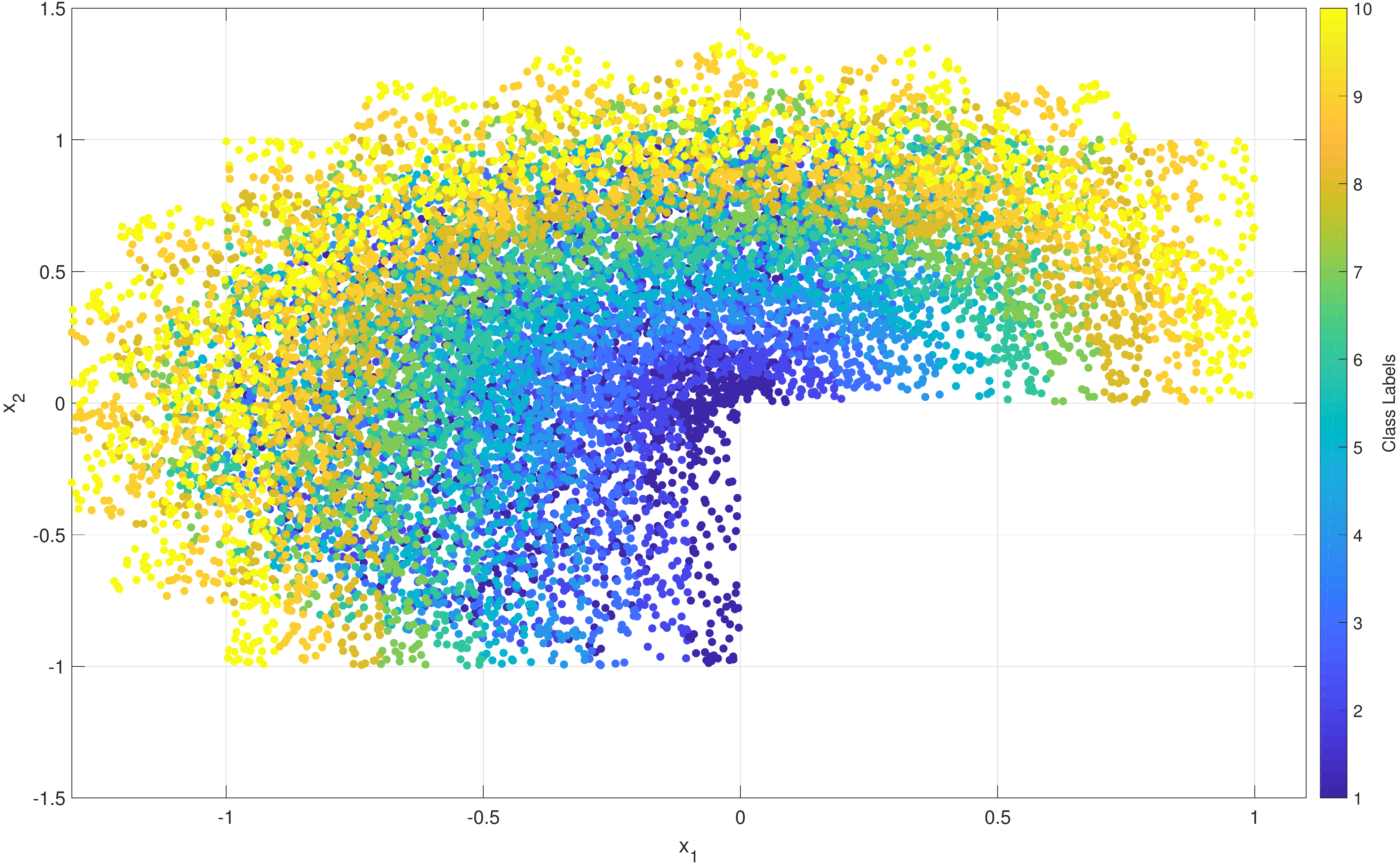}
     \caption{Synthetic Dataset: Thirteen tasks $ \theta = \{0,15,30,...,180 \}$ }
     \label{fig:synth2}
\end{minipage}
\end{figure}

\textbf{Satellite Dataset:} The problem is described in the introduction, and we used the dataset presented by \citet{sharma2015robust} modified for $c=3$ spacecraft.

\textbf{HAR Dataset:} This is a human activity recognition using smart-phone dataset from UCI repository \citep{anguita2013public}. Each of 30 volunteers performed six activities (walking, walking upstairs, walking downstairs, sitting, standing, laying) wearing the smart-phone.

\textbf{MNIST-MOD Dataset:} We randomly draw 1000 images from MNIST's train dataset. Then we rotate each of this image by randomly drawn angle from 0 to  180 degrees and repeat this 100 times to get data for 100 tasks. Examples for rotated MNIST dataset are shown in Fig. \ref{fig:MNISTROT}.

\begin{figure}[htb!]
  \centering
    \includegraphics[width=0.7\textwidth]{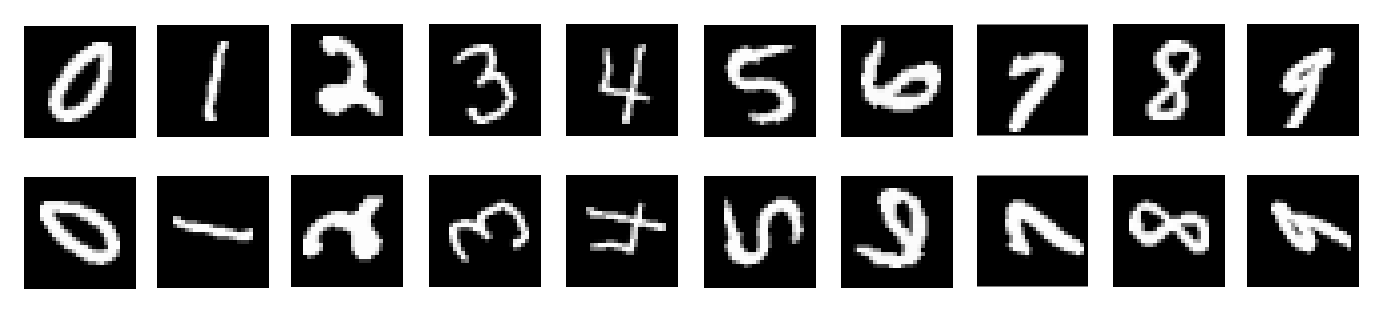}
     \caption{MNIST Data with no rotation (first row) and 90 degree rotation (second row)}
     \label{fig:MNISTROT}
\end{figure}

We use all Gaussian kernels and a novel random Fourier Feature (RFF) approximation, which extends the usual RFF approximation on Euclidean space $\sX $, \citet{rahimi2008random}, to the extended feature space $ \sP_{\sX} \times  \sX  $, to speed up the algorithm. We used Liblinear package for the implementation \citep{fan2008liblinear}.  All hyperparameters were selected using five fold cross-validation and experiments were repeated 10 times. We show results in Table \ref{tb:results} 
\footnote{The code to reproduce our results is available at \\ https://www.dropbox.com/sh/bls758ro5762mtf/AACbn3UXJItY9uwtmCAdi7E3a?dl=0}.

 The proposed method performs the best in three datasets and equally well in the one remaining dataset. The more our method outperforms pooling, the more knowledge can be shared between tasks. In MNIST data, our method perform marginally better than pooling. Note that we only use 1000 images for training and the accuracies can be improved further if one uses the entire MNIST training set. In case of synthetic and satellite dataset, the proposed method is significantly better. This is because similarity between tasks or domains is representative of similarity between predictors. In case of HAR dataset, proposed method doesn't improve pooling method and reason could be that all tasks are very similar to each other and pooling the data may be the best thing to do.

\begin{table}[htb!]
\centering
\begin{tabular}{ | c | c | c | }
\hline
Dataset & Pooling & Proposed Method   \\ \hline
Synthetic & 70.73 ( $ \pm 2.30 $)&  \textbf{25.40} ( $ \pm 1.72 $)  \\ \hline
Satellite &11.95  ( $ \pm 0.46$) & \textbf{ 8.28} ( $ \pm 0.79 $) \\ \hline
HAR & 1.69 ( $ \pm 0.56 $)  &  \textbf{1.68} ( $ \pm 0.58 $)\\ \hline
MNIST-MOD & 22.79 ( $ \pm 1.38 $)   & \textbf{21.39} ( $ \pm 1.24 $) \\ \hline
\end{tabular}
\caption{Percentage Error}
\label{tb:results}
\end{table}

\section{Conclusion and Future Work}
In this work, we extended the kernel-based algorithm for domain generalization of \citet{blanchard:11:nips} to the multiclass setting and proved the tightest known generalization error bound in terms of number of classes. We implemented the approach, demonstrating its improved performance w.r.t. a pooling strategy. Future work will focus on improved generalization bounds in terms of number of domains and extensions to zero shot learning. In extensions, we are interested in zero shot learning where training tasks have $c$ classes and test tasks have $c+1$ classes.

\newpage

\appendix

\numberwithin{equation}{section}
\numberwithin{theorem}{section}
\numberwithin{figure}{section}
\numberwithin{table}{section}
\renewcommand{\thesection}{{\Alph{section}}}
\renewcommand{\thesubsection}{\Alph{section}.\arabic{subsection}}
\renewcommand{\thesubsubsection}{\Roman{section}.\arabic{subsection}.\arabic{subsubsection}}
\setcounter{secnumdepth}{-1}
\setcounter{secnumdepth}{3}

\section{Proof of Theorem \ref{thm:EstErrCtrl}}
 Without loss of generality, it is assumed that \(n_i = n \). The function $\bmf$ can be split into c components $\bmf=(f_1,\ldots,f_c)$. We are interested in error bounds over $\bmf\in\fcal_{p,R}$
 \begin{align}
 \sup_{\bmf\in\fcal_{p,R}}| \widehat{\varepsilon}(\bmf) - {\varepsilon}(\bmf)|
    &\leq \sup_{\bmf\in\fcal_{p,R}}  \Big|\widehat{\varepsilon}(\bmf)  - \frac{1}{N}\sum_{i=1}^N \frac{1}{n_i}\sum_{j=1}^{n_i} \ell (\bmf({P}_{X}^i,{X}_{ij}), Y_{ij}) \Big| \notag\\
    & +  \sup_{\bmf\in\fcal_{p,R}} \Big|\frac{1}{N}\sum_{i=1}^N \frac{1}{n_i}\sum_{j=1}^{n_i} \ell (\bmf({P}_{X}^i,{X}_{ij}), Y_{ij}) - {\varepsilon}(\bmf)\Big|\notag\\
    & =: (I) + (II).\label{error-decomposition}
 \end{align}

 \subsection{Control of term I}
 For the first term, it follows from the Lipschitz continuity of $\ell$ that
 \begin{align}
 \Big|\widehat{\varepsilon}(\bmf) - \frac{1}{N}\sum_{i=1}^N \frac{1}{n_i}\sum_{j=1}^{n_i} \ell (\bmf({P}_{X}^i,{X}_{ij}), Y_{ij}) \Big|
    & = \Big|\frac{1}{N}\sum_{i=1}^{N}\frac{1}{n_i}\sum_{j=1}^{n_i}\Big(\ell\big(f(\hat{P}_X^i,X_{ij}),Y_{ij}\big)-\ell\big(f(P_X^i,X_{ij}),Y_{ij}\big)\Big)\Big|\notag\\
    & \leq \frac{1}{N}\sum_{i=1}^{N}\frac{L_\ell}{n_i}\sum_{j=1}^{n_i}\big\|f(\hat{P}_X^i,X_{ij})-f(P_X^i,X_{ij})\big\|_\infty\notag\\
    & = \frac{L_\ell}{N}\sum_{i=1}^{N}\frac{1}{n_i}\sum_{j=1}^{n_i}\max_{l=1,\ldots,c}\big|f_l(\hat{P}_X^i,X_{ij})-f_l(P_X^i,X_{ij})\big|\notag\\
    & \leq L_\ell\max_{l=1,\ldots,c}\max_{i=1,\ldots,N}\max_{j=1,\ldots,n_i}\big|f_l(\hat{P}_X^i,X_{ij})-f_l(P_X^i,X_{ij})\big|. \label{eq:first_term}
 \end{align}
 Now, let us look at the term \( | f_l(\widehat{P}_{X}^i,{X}_{ij}) -  f_l({P}_{X}^i,{X}_{ij}) |\) for some \(l \in \set{1,2,...,c}\).
 Using the reproducing property of the kernel, we derive
 \[
 | f_l(\widehat{P}_{X},X) -  f_l({P}_{X},X) | = |\langle \bar{k}\big( (\widehat{P}_{X},X), \cdot \big) - \bar{k}\big( ({P}_{X},X), \cdot \big), f_l \rangle |
    \leq  \| f_l \|_{\bar{k}}  \|  \bar{k}\big( (\widehat{P}_{X},X), \cdot \big) - \bar{k}\big( ({P}_{X},X) , \cdot \big) \|_{\bar{k}}.
 \]
 Furthermore, according to the H\"older continuity of the kernel $\kappa$, we know
 \begin{align*}
   \|\bar{k}\big( (\widehat{P}_{X},X), \cdot \big) - \bar{k}\big( ({P}_{X},X) , \cdot \big) \|_{\bar{k}} & = \Big( \bar{k}\big( (\widehat{P}_{X},X), (\widehat{P}_{X},X) \big)
            + \bar{k}\big( ({P}_{X},X), ({P}_{X},X) \big)
            -2  \bar{k}\big( (\widehat{P}_{X},X), ({P}_{X},X) \big) \Big)^{\frac{1}{2}}\\
    & = \sqrt{k(X,X)}\Big( \kappa\big( \Phi(\widehat{P}_{X}), \Phi(\widehat{P}_{X}) \big)
            + \kappa \big( \Phi({P}_{X}), \Phi({P}_{X}) \big)
            -2  \kappa\big( \Phi(\widehat{P}_{X}), \Phi({P}_{X}) \big) \Big)^{\frac{1}{2}}\\
    & \leq B_k \| \phi_{\kappa} (\Phi(P_{X})) -  \phi_{\kappa} (\Phi(\widehat{P}_{X})) \|_{\kappa} \leq B_kL_{\kappa} \| \Phi(P_{X}) -\Phi(\widehat{P}_{X}) \|_{k_{x'}}^{\alpha}.
 \end{align*}
 Therefore, there holds
 \[
 | f_l(\widehat{P}_{X},X) -  f_l({P}_{X},X) |\leq \| f_l \|_{\bar{k}} B_{k}L_{\kappa} \| \Phi(P_{X}) -\Phi(\widehat{P}_{X}) \|_{k_{x'}}^{\alpha}.
 \]
 We can bound $ \| \Phi(P_{X}) -\Phi(\widehat{P}_{X}) \|_{k_x'} $ using Hoeffding's inequality in the Hilbert space \(H_{k_x^\prime}\). Indeed, by Lemma \ref{thm:Hoeffding_Hilbert} with probability at least $1-\delta$ we have
 \begin{align*}
  \| \Phi(P_{X}) -\Phi(\widehat{P}_{X}) \|_{k'_x} & = \Big\|\frac{1}{n}\sum_{j=1}^{n}k'_x(X_j,\cdot)-\bE[k'_x(X,\cdot)]\Big\|_{k'_x}\\
  &\leq B_{k^\prime} \sqrt{\frac{2 \log(1/\delta)}{n}} + B_{k^\prime} \sqrt{\frac{1}{n}} + \frac{4 B_{k^\prime} \log(1/\delta)}{3n}.
  \end{align*}
 Combining the above two inequalities together, we derive the following inequality with at least probability $ 1 - \delta $ for all $j=1,\ldots,n_i$ and $l=1,\ldots,c$
 \begin{equation*}
     \label{eq:}
     \Big| f_l(\widehat{P}_{X}^i,{X}_{ij}) -  f_l({P}_{X}^i,{X}_{ij}) \Big| \leq \| f_l \|_{\bar{k}} B_{k} L_{\kappa} \Big( B_{k^\prime} \sqrt{\frac{2 \log(1/\delta)}{n}} + B_{k^\prime}  \sqrt{\frac{1}{n}} + \frac{4 B_{k^\prime}  \log(1/\delta)}{3n} \Big)^{\alpha}.
 \end{equation*}
 An union bound then implies the following inequality with at least probability $ 1 - \delta $ uniformly for all $i=1,\ldots,N,j=1,\ldots,n_i$ and $l=1,\ldots,c$
 \begin{equation}
     \label{eq:union_bound_f}
      \Big| f_l(\widehat{P}_{X}^i,{X}_{ij}) -  f_l({P}_{X}^i,{X}_{ij}) \Big| \leq \| f_l \|_{\bar{k}} B_{k} L_{\kappa}  \Big( B_{k^\prime} \sqrt{\frac{2 \log \frac{N}{\delta}}{n}} + B_{k^\prime} \sqrt{\frac{1}{n}} + \frac{4 B_{k^\prime} \log \frac{N}{\delta}}{3n} \Big)^{\alpha}.
 \end{equation}
Combining equation \ref{eq:first_term} and \ref{eq:union_bound_f} together, we derive the following inequality with at least probability $ 1 - \delta $
 \begin{equation}
 \label{eq:first_term_bound}
 \sup_{\bmf\in\fcal_{p,R}} \Big|\widehat{\varepsilon}(\bmf) - \frac{1}{N}\sum_{i=1}^N \frac{1}{n_i}\sum_{j=1}^{n_i} \ell (\bmf({P}_{X}^i,{X}_{ij}), Y_{ij}) \Big| \leq   L_{\ell} L_{\kappa} R B_{k}(B_{k^\prime})^{\alpha} \Big( \sqrt{\frac{2 \log \frac{N}{\delta}}{n}} + \sqrt{\frac{1}{n}} + \frac{4 \log \frac{N}{\delta}}{3n} \Big)^{\alpha},
 \end{equation}
 where we have used $\|f_l\|_{\bar{k}}\leq R$ for $\bmf\in\fcal_{p,R}$.

\subsection{Control of term II}
We now turn to the term II. To this aim, we first consider the term \((II)^\prime\) as the one sided version of term \((II)\) i.e.,
\begin{equation*}
    (II)^\prime := \sup_{\bmf\in\fcal_{p,R}} \frac{1}{N}\sum_{i=1}^N \frac{1}{n_i}\sum_{j=1}^{n_i} \ell (\bmf({P}_{X}^i,{X}_{ij}), Y_{ij}) - {\varepsilon}(\bmf),
\end{equation*}
which can be bounded by considering the following decomposition
    \begin{align}
        (II)^\prime & \leq
        \begin{multlined}[t]
                    \sup_{\bmf\in\fcal_{p,R}} \frac{1}{N}\sum_{i=1}^N \Big(\frac{1}{n_i}\sum_{j=1}^{n_i} \ell (\bmf({P}_{X}^i,{X}_{ij}), Y_{ij})
                        - \bE \bbrac{ \ell( \bmf( \widetilde{X}),Y ) \big| P_{XY}^i }\Big) \notag\\
                    + \sup_{\bmf\in\fcal_{p,R}} \frac{1}{N} \sum_{i=1}^N  \Big( \bE \bbrac{ \ell( \bmf( \widetilde{X}),Y ) \big| P_{XY}^i } - \bE \bbrac{ \ell( \bmf( \widetilde{X}),Y ) } \Big)
                \end{multlined}\notag\\
             & =: (IIa) + (IIb).\label{II-decomposition}
\end{align}

\paragraph{Control of Term (IIa).} Conditional to \(P_{XY}^1\),...,\(P_{XY}^N\) , the random variables \((X_{ij}, Y_{ij})_{ij}\) are independent (not identically distributed). Introduce a random variable
\begin{equation*}
    \zeta ((X_{ij}, Y_{ij})_{ij}) = \sup_{\bmf\in\fcal_{p,R}} \frac{1}{N}\sum_{i=1}^N \Big(\frac{1}{n_i}\sum_{j=1}^{n_i} \ell (\bmf({P}_{X}^i,{X}_{ij}), Y_{ij}) - \bE \bbrac{ \ell( \bmf( \widetilde{X}),Y ) \big| P_{XY}^i }\Big).
\end{equation*}
{
By the Lipschitz continuity of $\ell$, for any $\bmf\in\fcal_{p,R},P_X,X,Y$, we have
\begin{align*}
  \big|\ell(\bmf(P_X,X).Y)\big| & \leq \ell(0,Y)+\big|\ell(\bmf(P_X,X),Y)-\ell(0,Y)\big| \\
   & \leq B_Y + L_\ell\|\bmf(P_X,X)-0\|_\infty = B_Y+L_\ell\max_{m=1,\ldots,c}\big|\langle \bmf,\bar{k}((P_X,X),\cdot)\big| \\
   & \leq B_Y+L_\ell\max_{m=1,\ldots,c}\|\bmf\|_{\bar{k}}B_{\bar{k}}\leq B_Y+L_\ell RB_\kappa B_k:=B_\ell.
\end{align*} 
}
According to Azuma-McDiarmid's inequality~\citet{mcdiarmid1989method}, we derive the following inequality with probability at least \(1-\delta\) that
\begin{equation*}
    \zeta - \bE [\zeta | (P^i_{XY})_{1 \leq i \leq N}] \leq B_\ell\sqrt{\frac{\log(1/\delta)}{2Nn}}.
\end{equation*}
Next we bound the expectation term using standard Rademacher complexity analysis and then applying the extension to Talagrand's convex concentration inequality (see \citet{bartlett2002rademacher} theorem 7 and lemma 22). Let \((\epsilon_{ij})_{1 \leq i \leq N, 1 \leq j \leq n_i}\) be i.i.d Rademacher random variables.
\begin{equation*}
    \begin{aligned}
        \bE &[\zeta | (P^i_{XY})_{1 \leq i \leq N} ] \\
            &= \bE_{(X_{ij},Y_{ij})} \bbbrac{ \sup_{\bmf\in\fcal_{p,R}} \frac{1}{N}\sum_{i=1}^N \Big(\frac{1}{n_i}\sum_{j=1}^{n_i} \ell (\bmf({P}_{X}^i,{X}_{ij}), Y_{ij}) - \bE \bbrac{ \ell( \bmf( \widetilde{X}),Y ) \big| P_{XY}^i } \Big)\bigg| (P^i_{XY})_{1 \leq i \leq N} } \\
            &\leq \frac{2}{N} \bE_{(X_{ij},Y_{ij})} \bE_{(\epsilon_{ij})} \bbbrac{ \sup_{\bmf\in\fcal_{p,R}} \sum_{i=1}^N \frac{1}{n_i}\sum_{j=1}^{n_i} \epsilon_{ij} \ell (\bmf({P}_{X}^i,{X}_{ij}), Y_{ij}) \bigg| (P^i_{XY})_{1 \leq i \leq N} }.
    \end{aligned}
\end{equation*}
Recall that $n_1=\cdots=n_N=n$. The expectation of the term in the bracket over $\epsilon_{ij}$ is just the empirical Rademacher complexities of the class $\big\{(\widetilde{X},Y)\mapsto\ell(\bmf(\widetilde{X}),Y):\bmf\in\fcal_{p,R}\big\}$ w.r.t. the sample $\widetilde{S}:=\big\{(P_X^i,X_{ij},Y_{ij})\big\},i=1,\ldots,N,j=1,\ldots,n_i$..
We can apply Lemma \ref{lem:RC-MC}  with $m=Nn$ to show that
\begin{align*}
 &\frac{1}{Nn}\bE_{(\epsilon_{ij})} \bbbrac{ \sup_{\bmf\in\fcal_{p,R}} \sum_{i=1}^N \sum_{j=1}^{n_i} \epsilon_{ij} \ell (\bmf({P}_{X}^i,{X}_{ij}), Y_{ij}) \bigg| (P^i_{XY})_{1 \leq i \leq N} }\\
 & = \mathfrak{R}_{\widetilde{S}}\Big((P_X,X,Y)\mapsto\ell(\bmf(P_X,X),Y):\bmf\in\fcal_{p,R}\Big)\\
 & \leq 16L_\ell\sqrt{\log2}RB_\kappa B_k \big(Nn\big)^{-\frac{1}{2}}c^{\frac{1}{2}-\frac{1}{\max\{2,p\}}}\big(1+\log^{\frac{3}{2}}\sqrt{2}Nnc\big),
\end{align*}
where we have used $\sup_{P_X,X}\bar{k}((P_X,X),(P_X,X))\leq B_\kappa^2 B_{k}^2$.
It then follows that
\[
\bE [\zeta | (P^i_{XY})_{1 \leq i \leq N} ] \leq 27L_\ell RB_\kappa B_k \big(Nn\big)^{-\frac{1}{2}}c^{\frac{1}{2}-\frac{1}{\max\{2,p\}}}\big(1+\log^{\frac{3}{2}}\sqrt{2}Nnc\big),
\]
where we have used $32\sqrt{\log2}\leq27$.

\paragraph{Control of term (IIb).}
 Introduce a random variable,
\begin{equation*}
    \xi((P_{XY}^i)_{1 \leq i \leq N}) = \sup_{\bmf\in\fcal_{p,R}} \frac{1}{N} \sum_{i=1}^N  \Big( \bE \bbrac{ \ell( \bmf( \widetilde{X}),Y ) \big| P_{XY}^i } - \bE \bbrac{ \ell( \bmf( \widetilde{X}),Y ) } \Big).
\end{equation*}
Since \((P_{XY}^i)_{1 \leq i \leq N}\) are i.i.d we can apply Azuma-McDiarmid inequality~\citet{mcdiarmid1989method} to \(\xi\) to obtain the following inequality with probability $1-\delta$
\begin{equation*}
     \xi - \bE[\xi]  \leq B_{\ell} \sqrt{ \frac{\log(1/\delta)}{2N} }.
\end{equation*}

According to the standard symmetrization trick in relating the deviation between empirical means from expectations to Rademacher complexities, we also have
\begin{equation*}
    \begin{aligned}
        \bE[\xi] & =
        \begin{aligned}[t]
            \bE_{(P_{XY}^i)_{1 \leq i \leq N}} \bbbrac{ \sup_{\bmf\in\fcal_{p,R}} \frac{1}{N} \sum_{i=1}^N  \Big(\bE_{(X,Y) \sim P_{XY}^i}
                & \bbrac{ \ell( \bmf( \widetilde{X}),Y ) }
                - \bE_{ P_{XY} \sim \mu, (X,Y) \sim P_{XY} } \bbrac{ \ell( \bmf( \widetilde{X}),Y ) } \Big) }
        \end{aligned} \\
        & \leq \frac{2}{N} \bE_{(P_{XY}^i)_{1 \leq i \leq N}} \bE_{(\epsilon_i)_{1 \leq i \leq N}}
            \bbbrac{ \sup_{\bmf\in\fcal_{p,R}} \sum_{i=1}^N \epsilon_i \bE_{(X_i,Y_i) \sim P_{XY}^i} \bbrac{ \ell( \bmf( \widetilde{X}_i),Y_i ) } } \\
        & \leq \frac{2}{N} \bE_{(P_{XY}^i)_{1 \leq i \leq N}} \bE_{(X_i,Y_i) \sim P_{XY}^i} \bE_{(\epsilon_i)_{1 \leq i \leq N}}
            \bbbrac{ \sup_{\bmf\in\fcal_{p,R}} \sum_{i=1}^N \epsilon_i \bbrac{ \ell( \bmf( \widetilde{X}_i),Y_i ) } },
    \end{aligned}
\end{equation*}
where the last step is due to Jensen's inequality.
An application of Lemma \ref{lem:RC-MC} shows that
\[
\frac{1}{N}\bE_{(\epsilon_i)_{1 \leq i \leq N}}
            \bbbrac{ \sup_{\bmf\in\fcal_{p,R}} \sum_{i=1}^N \epsilon_i \bbrac{ \ell( \bmf( \widetilde{X}_i),Y_i ) } }
            \leq 16L_\ell\sqrt{\log2}RB_{\kappa}B_k N^{-\frac{1}{2}}c^{\frac{1}{2}-\frac{1}{\max\{2,p\}}}\big(1+\log^{\frac{3}{2}}\sqrt{2}Nc\big).
\]
Combining the above two inequalities together and using $32\sqrt{\log2}\leq27$, we derive the following inequality for $\mathbb{E}[\xi]$
\[
\bE[\xi]\leq 27L_\ell RB_{\kappa}B_k N^{-\frac{1}{2}}c^{\frac{1}{2}-\frac{1}{\max\{2,p\}}}\big(1+\log^{\frac{3}{2}}\sqrt{2}Nc\big).
\]

Combining terms (IIa) and (IIb), we derive the following inequality with probability at least $1-\delta$
\begin{align}
        (II)^\prime & \leq  27L_\ell RB_\kappa B_k c^{\frac{1}{2}-\frac{1}{\max\{2,p\}}}\Big(\big(Nn\big)^{-\frac{1}{2}}\big(1+\log^{\frac{3}{2}}c\sqrt{2}Nn\big)+N^{-\frac{1}{2}}\big(1+\log^{\frac{3}{2}}\sqrt{2}Nc\big)\Big) + 2B_{\ell}\sqrt{\frac{\log (2/\delta)}{2N}} \notag\\
             & \leq 54L_\ell RB_\kappa B_k c^{\frac{1}{2}-\frac{1}{\max\{2,p\}}}N^{-\frac{1}{2}}\big(1+\log^{\frac{3}{2}}c\sqrt{2}N\big) + 2B_{\ell}\sqrt{\frac{\log (2/\delta)}{2N}}.\label{eq:second_term_boundprime}
\end{align}
The bound for term \((II)\) can be obtained by replacing \(\delta\) with \(\delta/2\) as in standard Rademacher complexity analysis \citep{mohri2012foundations}. Therefore, we obtain the following inequality with probability at least \(1-\delta\)
\begin{equation}
    \label{eq:second_term_bound}
        (II) \leq 54L_\ell RB_\kappa B_k c^{\frac{1}{2}-\frac{1}{\max\{2,p\}}}N^{-\frac{1}{2}}\big(1+\log^{\frac{3}{2}}c\sqrt{2}N\big) + 2B_{\ell}\sqrt{\frac{\log (4/\delta)}{2N}}.
\end{equation}


\subsection{Combination of bounds}
Plugging the equations \ref{eq:first_term_bound}, \ref{eq:second_term_bound} back into \eqref{error-decomposition}, we derive the following inequality with probability at least $1-\delta$
\begin{multline}
    \sup_{f \in \fcal_{p,R}}| \widehat{\varepsilon}(\bmf) - {\varepsilon}(\bmf)| \leq
     L_{\ell} L_{\kappa} R B_{k}B_{k^\prime}^{\alpha} \bigg( \sqrt{\frac{2 \log \frac{2N}{\delta}}{n}} + \sqrt{\frac{1}{n}} + \frac{4 \log \frac{2N}{\delta}}{3n} \bigg)^{\alpha}\\
     + 54L_\ell RB_\kappa B_k c^{\frac{1}{2}-\frac{1}{\max\{2,p\}}}N^{-\frac{1}{2}}\big(1+\log^{\frac{3}{2}}c\sqrt{2}N\big) + 2B_{\ell}\sqrt{\frac{\log 8\delta^{-1}}{2N}},
 \end{multline}
 which can be written as the stated form.
The proof is complete.

\vskip 0.2in
\bibliographystyle{plainnat}

\bibliography{references}

\end{document}